\documentclass[11pt]{article}

\usepackage[margin=1in]{geometry}
\usepackage{amsmath,amssymb,amsthm,mathtools}
\usepackage{booktabs}
\usepackage{enumitem}
\usepackage{microtype}
\usepackage[dvipsnames]{xcolor}
\usepackage{tikz}
\usetikzlibrary{arrows.meta,positioning,calc}
\usepackage{placeins}
\usepackage[numbers,sort&compress]{natbib}
\usepackage[colorlinks=true,allcolors=MidnightBlue]{hyperref}
\usepackage[nameinlink,capitalise]{cleveref}

\hypersetup{
  pdftitle={Common Geodesics Do Not Guarantee Fisher Consistency of the Structured SVM: Minimal Counterexamples and a Tree-Metric Classification},
  pdfauthor={Jintao Fei and Jiangying Luo},
  pdfsubject={Fisher consistency of structured max-margin losses},
  pdfkeywords={structured SVM, Fisher consistency, max-margin loss, modular metric, tree metric}
}

\newtheorem{theorem}{Theorem}[section]
\newtheorem{proposition}[theorem]{Proposition}
\newtheorem{lemma}[theorem]{Lemma}
\newtheorem{corollary}[theorem]{Corollary}
\theoremstyle{definition}

\theoremstyle{remark}

\DeclareMathOperator*{\argmin}{arg\,min}
\DeclareMathOperator*{\argmax}{arg\,max}
\newcommand{\Y}{\mathcal Y}
\newcommand{\DeltaY}{\Delta_{\Y}}
\newcommand{\R}{\mathbb R}
\newcommand{\one}{\mathbf 1}
\newcommand{\Bayes}{\mathcal B}
\newcommand{\Vopt}{\mathcal V^{\star}}

\title{Common Geodesics Do Not Guarantee Fisher Consistency\\
of the Structured SVM: Minimal Counterexamples and\\
a Tree-Metric Classification}
\author{
Jintao Fei\\
\small JD.com\\
\small \texttt{fei-jintao@outlook.com}
\and
Jiangying Luo\\
\small Tsinghua University\\
\small \texttt{luo-jiangying@outlook.com}
}
\date{}

\begin{document}
\maketitle

\begin{abstract}
A known necessary condition for Fisher consistency of the structured support vector machine requires the task loss to be a metric for which every output triple has a common geodesic point.
We show that this condition is not sufficient for the canonical coordinate-wise argmax decoder.
A four-output unit star admits an exactly optimal score vector whose maximizers are all strictly non-Bayes, and four outputs are minimal among metrics satisfying the condition.
We then completely classify positively weighted tree metrics whose vertex set is the output space: argmax consistency holds if and only if the tree is a path.
The failure on branching trees is confined to boundary distributions; every tree retains the argmax property at every full-support distribution.
Among metrics satisfying the common-geodesic condition, five outputs are necessary and sufficient for a full-support counterexample; $K_{2,3}$ is the smallest member of an infinite $K_{m,n}$ family.
We additionally give a full-support counterexample for the three-dimensional Hamming cube.
All optimality claims have exact primal--dual certificates.
The counterexamples expose a concrete decoder gap: in this polyhedral setting, an embedding can guarantee the existence of a calibrated link without validating a prescribed argmax link on every surrogate-risk minimizer.
\end{abstract}

\section{Introduction}

The structured support vector machine (SSVM) extends the binary hinge loss to finite structured output spaces by loss-augmented inference \citep{crammer2001algorithmic,taskar2004max,tsochantaridis2005large}.
For a score vector $v\in\R^{|\Y|}$ and an observed output $y\in\Y$, its margin-rescaling surrogate is
\begin{equation}
  S_M(v,y)=\max_{y'\in\Y}\bigl\{L(y,y')+v_{y'}-v_y\bigr\}.
  \label{eq:ssvm}
\end{equation}
Despite its computational appeal, this surrogate need not be Fisher consistent beyond binary classification \citep{liu2007fisher,tewari2007consistency}.

Nowak--Vila, Rudi, and Bach \citep[Theorem~2.1]{nowak2022consistency} obtained a strong geometric necessary condition for consistency when $|\Y|>2$: $L$ must be a metric and every three outputs must share a point lying simultaneously on geodesics between all three pairs.
They explicitly left open whether this condition is sufficient.
Their Theorem~2.2 presented tree metrics as a positive class under the canonical decoder $v\mapsto\argmax_y v_y$.
The underlying Appendix Theorem~B.9 establishes the Bayes-risk identity and the embedded reports $-L_y$, but its final decoding step checks argmax only on those embedded reports.
Their consistency definition (Equation~(5)) selects an argmax prediction; we use the set-valued version \eqref{eq:strong-consistency}, but the distinction is immaterial to our negative results because every maximizer in the star and Hamming-cube witnesses is non-Bayes.

This paper shows that the omitted reports are essential.
Even when an SSVM embeds the target loss, a conditional surrogate risk can have additional minimizers outside the embedding image, and coordinate-wise argmax can decode those minimizers incorrectly.
Our first example is the smallest possible among metrics satisfying the common-geodesic condition: a tree on four vertices consisting of a center and three unit-length leaves.
The example is not a tie-breaking artefact---every score maximizer is a non-Bayes leaf while the center is the unique Bayes output.
We then prove an exact dichotomy: every branching weighted tree has the same obstruction, whereas every weighted path is argmax Fisher consistent for every distribution and every surrogate-risk minimizer.

Our contributions are:
\begin{enumerate}[leftmargin=2.2em]
  \item We disprove sufficiency of the common-geodesic condition with an exact four-output counterexample and prove that no three-output counterexample exists.
  \item We give a complete tree-metric classification: the canonical argmax rule is Fisher consistent exactly for paths.  We additionally prove that every tree is pointwise consistent at every full-support distribution, locating the branching obstruction precisely on the simplex boundary.
  \item Among metrics satisfying the common-geodesic condition, we prove a second sharp cardinality result: the smallest full-support counterexample has five outputs.  A four-point classification supplies the lower bound, and $K_{2,3}$ is the smallest member of an infinite complete-bipartite family attaining failure.  We also give a full-support counterexample for the familiar Hamming metric on $\{0,1\}^3$.
  \item We use surrogate level sets to pinpoint the missing condition between an embedding and a prescribed decoder.
\end{enumerate}

The counterexamples use rational probabilities and elementary score vectors.
Their optimality follows either from a telescoping cycle lower bound or from an explicit optimal-transport dual certificate.
Thus none of the conclusions relies on floating-point computation or finite-instance experimentation.

\paragraph{Related work.}
Consistency of multiclass and structured surrogates has a long history \citep{bartlett2006convexity,tewari2007consistency,osokin2017structured}.
General finite-loss calibration is studied by \citet{ramaswamy2016convex}, while recent work gives negative and positive guarantees for broader structured-max and structured comp-sum families \citep{mao2023structured}.
Those results do not settle the common-geodesic restriction isolated by \citet{nowak2022consistency}.
For the Crammer--Singer hinge, dominant-label restrictions are already necessary in ordinary multiclass classification \citep{liu2007fisher}.
Alternative structured surrogates with unconditional consistency include max-min constructions \citep{nowak2020maxmin}.
The embedding framework of \citet{finocchiaro2019embedding,finocchiaro2024embedding} proves that polyhedral embeddings give rise to calibrated links, but also emphasizes that extending the inverse map from embedded reports to all surrogate reports is a separate construction.
That abstract decoder distinction is known; our contribution is to instantiate it with exact SSVM counterexamples, sharp thresholds, and a tree-metric classification.
To the best of our knowledge, no counterexample to the sufficiency question or public correction of the canonical-argmax tree claim has previously been reported.

\section{Setting and exact optimality certificates}

Let $\Y$ be a finite set with $k=|\Y|$, and let $L:\Y\times\Y\to\R_{\ge0}$ be a metric.
For $q\in\DeltaY$, define the conditional surrogate and target risks
\begin{align}
  R_q(v)&=\sum_{y\in\Y}q_yS_M(v,y),
  &\Vopt(q)&=\argmin_{v\in\R^k}R_q(v),\\
  \ell_q(\hat y)&=\sum_{y\in\Y}q_yL(\hat y,y),
  &\Bayes(q)&=\argmin_{\hat y\in\Y}\ell_q(\hat y).
\end{align}
We write $H_M(q)=\min_vR_q(v)$ and
$H_L(q)=\min_{\hat y}\ell_q(\hat y)$ for the two Bayes risks.
We use a set-valued argmax, and the consistency notion in question is
\begin{equation}
  v\in\Vopt(q)
  \quad\Longrightarrow\quad
  \argmax_{y\in\Y}v_y\subseteq\Bayes(q)
  \qquad\text{for every }q\in\DeltaY.
  \label{eq:strong-consistency}
\end{equation}
Our counterexamples in fact have disjoint argmax and Bayes sets, so they also invalidate every deterministic tie-breaking rule based on coordinate-wise maximization.

\subsection{The common-geodesic condition}

For $x,y\in\Y$, write
\[
  I(x,y)=\{z\in\Y:L(x,y)=L(x,z)+L(z,y)\}
\]
for the metric interval between $x$ and $y$.
The necessary condition of \citet{nowak2022consistency} is
\begin{equation}
  I(y_1,y_2)\cap I(y_1,y_3)\cap I(y_2,y_3)\ne\varnothing
  \qquad\text{for every }y_1,y_2,y_3\in\Y.
  \label{eq:cg}
\end{equation}
In metric geometry, \eqref{eq:cg} is the defining property of a \emph{modular metric space}; for graph shortest-path metrics it corresponds to modularity of the graph \citep{bandelt1993modular,bandelt2008metric,karzanov2004one}.
Median metrics impose uniqueness of the common point.
Every tree metric satisfies \eqref{eq:cg}: the three pairwise paths meet at a common tree median.
Hamming cubes are also median, whereas complete bipartite graphs $K_{m,n}$ are modular but generally not median.

\subsection{A transport dual for the conditional SSVM risk}

The following standard linear-programming representation will certify our full-support examples.
It is also the finite optimal-transport form of the SSVM Bayes risk used by \citet{nowak2022consistency}.

\begin{proposition}[Self-coupling dual]
For every $q\in\DeltaY$,
\begin{equation}
  \min_{v\in\R^k}R_q(v)
  =\max_{\Pi\ge0}
    \sum_{i,j\in\Y}\Pi_{ij}L(i,j)
  \quad\text{subject to}\quad
  \Pi\one=q,\quad \Pi^\top\one=q.
  \label{eq:transport-dual}
\end{equation}
If $(v,\xi)$ and $\Pi$ are primal and dual optimal, then
\begin{equation}
  \Pi_{ij}>0
  \quad\Longrightarrow\quad
  \xi_i=L(i,j)+v_j-v_i,
  \qquad \xi_i=S_M(v,i).
  \label{eq:cs}
\end{equation}
\end{proposition}

\begin{proof}
Introduce one epigraph variable $\xi_i$ per observed output and write
\begin{equation}
  \min_{\xi,v}\sum_iq_i\xi_i
  \quad\text{subject to}\quad
  \xi_i+v_i-v_j\ge L(i,j),\qquad i,j\in\Y.
  \label{eq:primal-lp}
\end{equation}
The variables $\xi_i$ may be treated as free because the constraints with $i=j$ imply $\xi_i\ge0$.
Assigning nonnegative multipliers $\Pi_{ij}$ to the constraints, stationarity in $\xi$ gives $\Pi\one=q$, while stationarity in $v$ gives equality of the row and column marginals, hence $\Pi^\top\one=q$.
The dual objective is the right-hand side of \eqref{eq:transport-dual}.
Both programs are feasible and finite, so linear-programming strong duality applies.
Equation \eqref{eq:cs} is complementary slackness.
\end{proof}

Two elementary consequences will be used repeatedly.

\begin{lemma}[Canonical embedded reports]
\label{lem:canonical-report}
For every finite metric $L$, define $\phi(y)=-L_y$ by
$\phi(y)_z=-L(y,z)$.
Then
\begin{equation}
  S_M(\phi(y),x)=2L(y,x)
  \qquad(x,y\in\Y).
  \label{eq:embedded-loss}
\end{equation}
\end{lemma}

\begin{proof}
The triangle inequality gives
\[
  S_M(\phi(y),x)
  =\max_z\{L(x,z)-L(y,z)+L(y,x)\}
  \le2L(y,x),
\]
and equality is attained at $z=y$.
\end{proof}

\begin{lemma}[Finite supported-coupling criterion]
\label{lem:hall}
Let $\mu$ and $\nu$ be nonnegative measures of equal total mass on finite sets $X$ and $Z$, and let $E\subseteq X\times Z$ be an allowed support.
For $A\subseteq X$, write
$N(A)=\{z\in Z:(x,z)\in E\text{ for some }x\in A\}$.
There exists a coupling with marginals $\mu,\nu$ supported on $E$ if and only if
\begin{equation}
  \mu(A)\le\nu(N(A))\qquad\text{for every }A\subseteq X.
  \label{eq:weighted-hall}
\end{equation}
If $X=Z$, $\mu=\nu$, and $E$ is symmetric, the coupling may be chosen symmetric.
\end{lemma}

\begin{proof}
Necessity follows by summing the coupling over rows in $A$.
Sufficiency is the max-flow--min-cut theorem applied to the bipartite network with source capacities $\mu$, infinite capacities on $E$, and sink capacities $\nu$.
In the symmetric case, averaging a feasible coupling with its transpose preserves its marginals and support.
\end{proof}

\section{A complete classification for tree metrics}

Throughout this section, a tree metric means the shortest-path metric of a positively weighted tree whose vertex set is exactly $\Y$; no unobserved Steiner vertices are allowed.
We first isolate a metric configuration that forces failure, and then prove that its absence is sufficient within this class.

\subsection{The tripod obstruction}

\begin{theorem}[Tripod obstruction]
\label{thm:tripod}
Suppose a finite metric space contains four distinct points $z,y_1,y_2,y_3$ such that, with $r_i=L(z,y_i)>0$,
\begin{equation}
  L(y_i,y_j)=r_i+r_j\qquad(i\ne j).
  \label{eq:tripod}
\end{equation}
Then the max-margin loss \eqref{eq:ssvm} is not argmax Fisher consistent.
\end{theorem}

\begin{proof}
Put $q_{y_i}=1/3$ and assign zero mass to every other output.
Let $r_\circ=\min_i r_i$, fix $\varepsilon>0$, and set
\begin{align*}
  v_{y_i}&=-r_i,\\
  v_x&=\min\left\{-r_\circ-\varepsilon,
          \min_i\bigl(r_i-L(y_i,x)\bigr)\right\},
  \qquad x\notin\{y_1,y_2,y_3\}.
\end{align*}
For the true output $y_i$, either of the other selected points $y_j$ gives the loss-augmented value $2r_i$ by \eqref{eq:tripod}.
The definition of every remaining coordinate gives
$L(y_i,x)+v_x-v_{y_i}\le2r_i$; the selected coordinates obey the same bound.
Consequently,
\[
  S_M(v,y_i)=2r_i,
  \qquad R_q(v)=\frac23(r_1+r_2+r_3).
\]
For arbitrary scores $u$, use the competitors $y_2,y_3,y_1$ cyclically in the three maxima.
The score differences telescope, yielding
\[
  3R_q(u)\ge
  L(y_1,y_2)+L(y_2,y_3)+L(y_3,y_1)
  =2(r_1+r_2+r_3).
\]
Thus $v\in\Vopt(q)$.

Adding the triangle inequalities for the three pairs gives, for every $x\in\Y$,
\[
  \sum_{i=1}^3L(x,y_i)\ge r_1+r_2+r_3,
\]
so $z$ is Bayes optimal.
On the other hand,
\[
  \ell_q(y_i)=\ell_q(z)+\frac{r_i}{3}>\ell_q(z).
\]
Finally, $\argmax v=\{y_i:r_i=r_\circ\}$, and hence every score maximizer is non-Bayes.
\end{proof}

Every vertex of degree at least three in a positively weighted tree supplies \eqref{eq:tripod}: choose one vertex from each of three components obtained after removing the branching vertex.
The smallest instance is the unit star in \cref{fig:examples}.

\begin{corollary}[Four-output star]
\label{cor:star}
Let $\Y=\{o,a,b,c\}$ with
\[
  L(o,a)=L(o,b)=L(o,c)=1,
  \qquad L(a,b)=L(b,c)=L(c,a)=2.
\]
For
\[
  q=(q_o,q_a,q_b,q_c)=\left(0,\frac13,\frac13,\frac13\right),
  \qquad
  v=(v_o,v_a,v_b,v_c)=(-1,0,0,0),
\]
one has $v\in\Vopt(q)$, $\argmax v=\{a,b,c\}$, and $\Bayes(q)=\{o\}$.
\end{corollary}

\begin{proof}
For each leaf $y$, $S_M(v,y)=2$, so $R_q(v)=2$.
For every $u$,
\[
  S_M(u,a)+S_M(u,b)+S_M(u,c)
  \ge (2+u_b-u_a)+(2+u_c-u_b)+(2+u_a-u_c)=6,
\]
which proves optimality.
The target risks are $\ell_q(o)=1$ and
$\ell_q(a)=\ell_q(b)=\ell_q(c)=4/3$.
\end{proof}

\subsection{Weighted paths are consistent}

\begin{theorem}[Path consistency]
\label{thm:path}
Let $L$ be the shortest-path metric on an arbitrary finite positively weighted path.
Then \eqref{eq:strong-consistency} holds for every $q\in\DeltaY$, including boundary distributions and Bayes ties.
\end{theorem}

\begin{proof}
Order the vertices as $x_1<\cdots<x_n$, put
$w_t=L(x_t,x_{t+1})>0$, and write $Q_t=\sum_{i\le t}q_i$, with $Q_0=0$.
A direct edge calculation gives
\begin{equation}
  \ell_q(x_{t+1})-\ell_q(x_t)=w_t(2Q_t-1).
  \label{eq:path-risk-difference}
\end{equation}
For any Bayes vertex $x_m$, the embedded score
$v_j=-L(x_j,x_m)$ satisfies $S_M(v,x_i)=2L(x_i,x_m)$.
Thus the optimal surrogate risk is at most $2\ell_q(x_m)$.

First suppose $Q_{m-1}<1/2<Q_m$, so $x_m$ is the unique Bayes vertex by \eqref{eq:path-risk-difference}.
Let $p=\sum_{i<m}q_i$ and $r=\sum_{i>m}q_i$.
Assume $p\ge r$; the other case is symmetric.
Choose a submeasure $\beta_i\le q_i$ on $i<m$ with $\sum_{i<m}\beta_i=r$, and couple $\beta$ to the restriction of $q$ on $j>m$ using a transport plan $\gamma$.
With $\alpha_i=q_i-\beta_i$, define the symmetric self-coupling
\begin{align*}
  \Pi_{ij}=\Pi_{ji}&=\gamma_{ij} &&(i<m<j),\\
  \Pi_{im}=\Pi_{mi}&=\alpha_i &&(i<m),\\
  \Pi_{mm}&=1-2p>0,
\end{align*}
and set all other entries to zero.
Its two marginals are $q$.
Every pair in its support has a path through $x_m$, so its objective in \eqref{eq:transport-dual} is $2\ell_q(x_m)$.
It is therefore dual optimal.
For an arbitrary primal optimum, complementary slackness at $(m,m)$ gives $\xi_m=0$.
Primal feasibility then implies
\[
  v_m-v_k\ge L(x_m,x_k)>0\qquad(k\ne m),
\]
so $x_m$ is the unique score maximizer.

It remains to handle a median plateau.
If some $Q_t=1/2$, let
\[
  a=\min\{t:Q_t=1/2\},
  \qquad b=\min\{t>a:Q_t>1/2\}.
\]
Then \eqref{eq:path-risk-difference} gives
$\Bayes(q)=\{x_a,\ldots,x_b\}$, with $q_a,q_b>0$ and zero mass on the intervening vertices.
The sets $A=\{i:i\le a\}$ and $B=\{j:j\ge b\}$ each have mass $1/2$.
Define
\[
  \Pi_{ij}=\Pi_{ji}=2q_iq_j\qquad(i\in A,\ j\in B),
\]
with all other entries zero.
This is a self-coupling with value $2\ell_q(x_a)$, hence it is optimal, and $\Pi_{ab},\Pi_{ba}>0$.
Tightness of the row-$b$, column-$a$ constraint, compared with the row-$b$, column-$k$ constraint for $k<a$, gives
\[
  v_a-v_k\ge L(x_b,x_k)-L(x_b,x_a)=L(x_a,x_k)>0.
\]
Using the active arc $(a,b)$ symmetrically gives $v_b>v_k$ for every $k>b$.
No output outside the Bayes plateau can therefore maximize $v$.
\end{proof}

\begin{corollary}[Exact tree-metric classification]
\label{cor:tree-classification}
For the shortest-path metric of a finite positively weighted tree, the max-margin surrogate is argmax Fisher consistent in the sense of \eqref{eq:strong-consistency} if and only if the tree is a path.
\end{corollary}

\begin{proof}
Paths are covered by \cref{thm:path}.
Every non-path tree has a vertex of degree at least three and therefore contains the tripod of \cref{thm:tripod}.
\end{proof}

The preceding failure is genuinely a boundary phenomenon for tree metrics.

\begin{theorem}[All tree metrics are consistent in the simplex interior]
\label{thm:tree-interior}
Let $L$ be the shortest-path metric of any finite positively weighted tree.
For every full-support $q\in\DeltaY$ and every $v\in\Vopt(q)$,
$\argmax v\subseteq\Bayes(q)$.
\end{theorem}

\begin{proof}
For an oriented edge $e=(u,w)$ of length $c_e$, let $C_w$ be the component containing $w$ after removing $e$.
Then
\begin{equation}
  \ell_q(w)-\ell_q(u)=c_e\bigl(1-2q(C_w)\bigr).
  \label{eq:tree-edge-risk}
\end{equation}
It follows that the tree median is either a unique vertex $m$, for which every component of $T-m$ has mass strictly below $1/2$, or two adjacent vertices $a,b$ whose connecting edge splits the tree into sets $A,B$ of mass $1/2$ each.
Indeed, \eqref{eq:tree-edge-risk} makes the risk unimodal along every path, and full support rules out a risk-flat segment containing two consecutive edges.

In the first case, let $p_s<1/2$ be the masses of the components of $T-m$.
Choose
\[
  0<\delta<\min\left\{q_m,1-2\max_s p_s\right\}.
\]
Let $\mu=q-\delta e_m$, whose total mass is $M=1-\delta$, and allow a pair $(x,y)$ exactly when at least one of $x,y$ equals $m$, or $x$ and $y$ lie in different components of $T-m$.
The supported-coupling criterion in \cref{lem:hall} supplies a coupling $\Pi'$ of $\mu$ to itself on this support.
Indeed, a row set meeting two components, or containing $m$, has the full column neighborhood; a row set contained in one component $C_s$ has mass at most $p_s$ and neighborhood mass
\[
  \mu(\Y\setminus C_s)=M-p_s>p_s.
\]
Now set
\[
  \Pi=\frac12\bigl(\Pi'+(\Pi')^\top\bigr)
       +\delta e_me_m^\top.
\]
Every supported pair has a path through $m$, so the dual value is $2\ell_q(m)$, matching the embedded-report upper bound.
Since $\Pi_{mm}>0$, complementary slackness gives $\xi_m=0$, which forces $v_m>v_x$ for every $x\ne m$.

In the second case, use the cross-cut coupling
\[
  \Pi_{ij}=\Pi_{ji}=2q_iq_j\qquad(i\in A,\ j\in B).
\]
Every supported path crosses $(a,b)$, so its value is
$2\ell_q(a)=2\ell_q(b)$ and it is dual optimal.
Full support gives $\Pi_{ab},\Pi_{ba}>0$.
Comparing the active row-$b$, column-$a$ constraint with any column $x\in A\setminus\{a\}$ yields $v_a>v_x$; symmetrically, $v_b>v_x$ for $x\in B\setminus\{b\}$.
Thus every score maximizer is one of the two Bayes vertices.
\end{proof}

\FloatBarrier
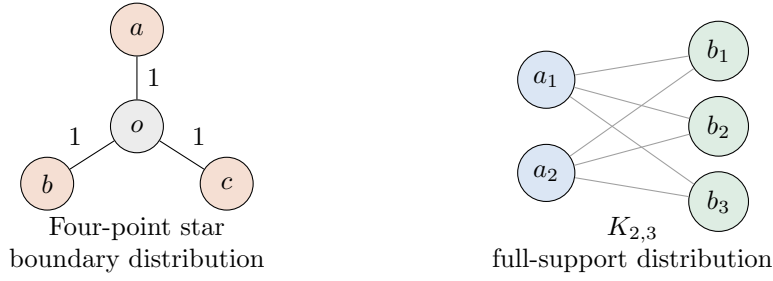
\begin{figure}[t]
\centering
\begin{tikzpicture}[scale=0.95, every node/.style={font=\small}]
  \begin{scope}[xshift=-3.6cm]
    \node[circle,draw,fill=gray!15,minimum size=7mm] (o) at (0,0) {$o$};
    \node[circle,draw,fill=BrickRed!12,minimum size=7mm] (a) at (0,1.35) {$a$};
    \node[circle,draw,fill=BrickRed!12,minimum size=7mm] (b) at (-1.25,-0.8) {$b$};
    \node[circle,draw,fill=BrickRed!12,minimum size=7mm] (c) at (1.25,-0.8) {$c$};
    \draw (o)--node[right] {$1$}(a);
    \draw (o)--node[above left] {$1$}(b);
    \draw (o)--node[above right] {$1$}(c);
    \node[align=center] at (0,-1.65) {Four-point star\\boundary distribution};
  \end{scope}
  \begin{scope}[xshift=3.3cm]
    \node[circle,draw,fill=RoyalBlue!12,minimum size=7mm] (a1) at (-1.2,0.65) {$a_1$};
    \node[circle,draw,fill=RoyalBlue!12,minimum size=7mm] (a2) at (-1.2,-0.65) {$a_2$};
    \node[circle,draw,fill=ForestGreen!12,minimum size=7mm] (b1) at (1.2,1.05) {$b_1$};
    \node[circle,draw,fill=ForestGreen!12,minimum size=7mm] (b2) at (1.2,0) {$b_2$};
    \node[circle,draw,fill=ForestGreen!12,minimum size=7mm] (b3) at (1.2,-1.05) {$b_3$};
    \foreach \x in {a1,a2}\foreach \y in {b1,b2,b3}\draw[gray!75] (\x)--(\y);
    \node[align=center] at (0,-1.65) {$K_{2,3}$\\full-support distribution};
  \end{scope}
\end{tikzpicture}
\caption{Sharp cardinality witnesses among metrics satisfying the common-geodesic condition: the smallest unrestricted counterexample (left) and the smallest full-support counterexample (right).  Edge lengths are one and distances are shortest-path distances.}
\label{fig:examples}
\end{figure}
\FloatBarrier

\section{Sharp lower bounds on output cardinality}

We now prove that \cref{cor:star} has the smallest possible output space among all metrics satisfying \eqref{eq:cg}, not merely among tree metrics.

\begin{lemma}[Three points form a weighted path]
\label{lem:three-path}
Let $|\Y|=3$ and suppose $L$ is a metric satisfying \eqref{eq:cg}.
After relabeling $\Y=\{1,2,3\}$, there exist $A,B>0$ such that
\[
  L(1,2)=A,\qquad L(2,3)=B,\qquad L(1,3)=A+B.
\]
\end{lemma}

\begin{proof}
Apply \eqref{eq:cg} to the three distinct outputs.
The common geodesic point must be one of them; if it is $2$, the only nontrivial equality is $L(1,3)=L(1,2)+L(2,3)$.
The other cases are relabelings.
\end{proof}

\begin{corollary}[Minimum unrestricted cardinality]
\label{cor:min-unrestricted}
Among metrics satisfying \eqref{eq:cg}, the minimum cardinality of an argmax Fisher-inconsistency counterexample is four.
\end{corollary}

\begin{proof}
By \cref{lem:three-path,thm:path}, every admissible three-point metric is consistent.
The four-point unit star in \cref{cor:star} is inconsistent.
\end{proof}

We next determine the sharp threshold when the counterexample distribution is required to have full support.

\begin{lemma}[Classification of four-point common-geodesic metrics]
\label{lem:four-classification}
Every four-point metric satisfying \eqref{eq:cg} is isometric to one of the following:
\begin{enumerate}[label=(\roman*),leftmargin=2.2em]
  \item a positively weighted three-leaf star whose center is one of the four points;
  \item a positively weighted four-vertex path; or
  \item a weighted rectangle $\{00,10,01,11\}$ with
  \begin{equation}
    L(x,y)=A|x_1-y_1|+B|x_2-y_2|,
    \qquad A,B>0.
    \label{eq:rectangle}
  \end{equation}
\end{enumerate}
\end{lemma}

\begin{proof}
If there exists a triple for which the omitted fourth output is also a common geodesic point, the three equalities in \eqref{eq:cg} give a weighted three-leaf star immediately.
Otherwise every metric triangle is degenerate.
Choose a diameter pair $x_1,x_4$ of length $D$.
For each remaining point $x_i$, degeneracy and maximality of $D$ force
\[
  D=L(x_1,x_i)+L(x_i,x_4),\qquad i=2,3.
\]
Put $s=L(x_1,x_2)$ and $t=L(x_1,x_3)$, and assume $s\le t$.
Degeneracy of the first of the triangles $(x_1,x_2,x_3)$ and $(x_4,x_2,x_3)$ gives
\[
  L(x_2,x_3)\in\{t-s,s+t\},
\]
while degeneracy of the second gives
\[
  L(x_2,x_3)\in\{t-s,2D-s-t\}.
\]
The two crossed choices would force $s=0$ or $t=D$, contradicting distinctness.
Hence either
\[
  L(x_2,x_3)=t-s,
\]
which places all four points on a path at locations $0,s,t,D$, or
\[
  L(x_2,x_3)=s+t=2D-s-t.
\]
In the latter case $s+t=D$ and $L(x_2,x_3)=D$.
Under the labeling $x_1=00,x_2=10,x_3=01,x_4=11$, the remaining four distances alternate between $s$ and $t$, giving \eqref{eq:rectangle} with weights $s,t$.
These alternatives also cover $s=t$; distinctness then forces the rectangle case.
\end{proof}

\begin{proposition}[Full-support consistency of the weighted rectangle]
\label{prop:rectangle-interior}
For the metric \eqref{eq:rectangle}, every full-support distribution $q$ and every $v\in\Vopt(q)$ satisfy
$\argmax v\subseteq\Bayes(q)$.
\end{proposition}

\begin{proof}
Write $p_c=\Pr_q(Y_c=1)$.
The target Bayes set is obtained by taking a majority bit separately in each coordinate.
For every self-coupling $(X,Y)$ and $c\in\{1,2\}$,
\begin{equation}
  \Pr(X_c\ne Y_c)\le2\min(p_c,1-p_c).
  \label{eq:rectangle-cut-bound}
\end{equation}

First suppose both coordinate majorities are strict.
After flipping bits, take $p_1,p_2<1/2$, so $00$ is the unique Bayes output.
Choose
\[
  0<\delta<\min\{q_{00},q_{00}-q_{11},1-2p_1,1-2p_2\};
\]
the second entry is positive because $q_{00}-q_{11}=1-p_1-p_2$.
Remove mass $\delta$ from both $00$ marginals and allow a pair $(x,y)$ exactly when it never has $x_c=y_c=1$.
The four row neighborhoods are
\[
  N(00)=\Y,\quad N(10)=\{00,01\},\quad
  N(01)=\{00,10\},\quad N(11)=\{00\}.
\]
For row sets not containing $00$, the only critical Hall checks are the sets
$\{11\}$, $\{10,11\}$, $\{01,11\}$, and $\{10,01,11\}$; the remaining singleton and two-element checks follow from these and the $00$ mass bound.
By \cref{lem:hall}, a supported coupling exists because the weighted Hall inequalities reduce to
\[
  \delta\le q_{00},\quad
  \delta\le q_{00}-q_{11},\quad
  \delta\le1-2p_1,\quad
  \delta\le1-2p_2,
\]
so an allowed coupling exists.
Add mass $\delta$ at $(00,00)$ and symmetrize.
The resulting self-coupling attains both bounds in \eqref{eq:rectangle-cut-bound}, hence is dual optimal, and has $\Pi_{00,00}>0$.
Complementary slackness forces $00$ to be the unique score maximizer.

Suppose instead that $p_1=1/2$ and $p_2<1/2$; all other one-tie cases are symmetric.
Now $\Bayes(q)=\{00,10\}$ and
$g=q_{10}-q_{01}=1/2-p_2>0$.
Define a self-coupling by the six nonzero entries
\begin{align*}
  \Pi_{00,10}=\Pi_{10,00}&=g,&
  \Pi_{00,11}=\Pi_{11,00}&=q_{11},&
  \Pi_{01,10}=\Pi_{10,01}&=q_{01}.
\end{align*}
Its marginals are $q$, and it attains equality in \eqref{eq:rectangle-cut-bound} for both coordinates, so it is dual optimal.
Tightness of the two active constraints in row $10$ and then in row $00$ gives
\[
  v_{00}=v_{01}+B,
  \qquad v_{10}=v_{11}+B.
\]
Thus neither non-Bayes output can maximize the score.
If both coordinates tie, all four outputs are Bayes.
\end{proof}

\begin{theorem}[No four-output full-support obstruction]
\label{thm:four-interior}
Every metric with $3\le|\Y|\le4$ satisfying \eqref{eq:cg} obeys the argmax property at every full-support distribution.
\end{theorem}

\begin{proof}
The three-point case follows from \cref{lem:three-path,thm:path}.
For four points, use \cref{lem:four-classification}.
The star and path cases follow from \cref{thm:tree-interior}, and the rectangle case follows from \cref{prop:rectangle-interior}.
\end{proof}

\section{Full-support counterexamples}

One might hope that the four-point obstruction is caused solely by $q_o=0$.
The next examples show that restricting to the relative interior of the probability simplex does not rescue sufficiency of \eqref{eq:cg}.

\subsection{An infinite complete-bipartite family}

\begin{theorem}[Complete-bipartite obstruction]
\label{thm:complete-bipartite}
Let $2\le m<n$, and equip the vertices $\Y=A\sqcup B$ of $K_{m,n}$ with unweighted graph distance, where $|A|=m$ and $|B|=n$.
Then $L$ satisfies \eqref{eq:cg}, but $S_M$ fails argmax Fisher consistency at the uniform, full-support distribution.
\end{theorem}

\begin{proof}
Distinct vertices in the same part are at distance two, while vertices in opposite parts are at distance one.
For a triple containing points from both parts, its singleton part supplies a common geodesic point.
For three points in one part, any vertex in the other part does so.
Thus \eqref{eq:cg} holds.

Let $N=m+n$, put $q_y=1/N$, and set $v_a=-1$ for $a\in A$ and $v_b=0$ for $b\in B$.
For a true label in $A$, either another $A$ vertex or any $B$ vertex yields loss-augmented value two.
For a true label in $B$, another $B$ vertex yields two.
No candidate yields more, so $S_M(v,y)=2$ for all $y$ and $R_q(v)=2$.

For a dual certificate, choose a directed Hamilton cycle within each part and put mass $1/N$ on every cycle arc.
Both marginals are uniform; all $N$ used arcs have distance two.
The dual value is two, proving $v\in\Vopt(q)$.
For $a\in A$ and $b\in B$, respectively,
\[
  \ell_q(a)=\frac{2m+n-2}{N},
  \qquad
  \ell_q(b)=\frac{m+2n-2}{N}.
\]
Since $m<n$, $\Bayes(q)=A$, whereas $\argmax v=B$.
\end{proof}

The smallest member of this family is the graph $K_{2,3}$ shown in \cref{fig:examples}.

\begin{corollary}[Minimum full-support cardinality]
\label{cor:min-full-support}
Among metrics satisfying \eqref{eq:cg}, the minimum cardinality of a counterexample at a full-support distribution is five.
\end{corollary}

\begin{proof}
The lower bound is \cref{thm:four-interior}; the $K_{2,3}$ instance of \cref{thm:complete-bipartite} attains it.
\end{proof}

\subsection{A full-support Hamming-cube counterexample}

The preceding family settles the full-support cardinality threshold.
The next construction is useful for a different reason: it shows the failure for the standard decomposable Hamming metric, with a unique Bayes output and an argmax set disjoint from it.

\begin{theorem}[Three-dimensional Hamming cube]
\label{thm:cube}
Let $\Y=\{0,1\}^3$ with Hamming distance and lexicographic ordering
\[
000,001,010,011,100,101,110,111.
\]
Set
\[
q=\frac1{11}(2,1,1,1,1,2,2,1),
\]
and give even-parity vertices score zero and odd-parity vertices score $-1$.
Then this score vector minimizes $R_q$, but its argmax is disjoint from the unique Hamming Bayes output.
\end{theorem}

\begin{proof}
The Hamming cube satisfies \eqref{eq:cg}: the coordinate-wise majority of any three vertices lies on a shortest path between every pair.

Let $E=\{000,011,101,110\}$ be the even-parity class and $O=\Y\setminus E$.
For $y\in E$, an even vertex at distance two or the odd antipode at distance three gives $S_M(v,y)=2$.
For $y\in O$, the even antipode gives $S_M(v,y)=3+1=4$.
Since $q(E)=7/11$ and $q(O)=4/11$,
\[
  R_q(v)=\frac{7\cdot2+4\cdot4}{11}=\frac{30}{11}.
\]

Construct a dual coupling as follows.
Put mass $1/11$ on both directions of each of the four antipodal pairs, and mass $1/11$ on the three-cycle
\[
  000\longrightarrow101\longrightarrow110\longrightarrow000.
\]
The antipodal arcs supply one unit of incoming and outgoing mass to every vertex; the cycle supplies the second unit exactly to the three vertices having numerator two in $q$.
Thus both marginals equal $q$.
Its objective value is
\[
  \frac{8\cdot3+3\cdot2}{11}=\frac{30}{11},
\]
so $v$ is primal optimal.

The probabilities that coordinates one, two, and three equal one are respectively
$6/11,5/11,5/11$.
Hamming risk is minimized coordinate-wise, hence the unique Bayes output is $100$.
This vertex has odd parity, while $\argmax v=E$.
\end{proof}

\section{Where the embedding-to-argmax implication fails}

Recall from \cref{lem:canonical-report} that the canonical reports
$\phi(y)=-L_y$ satisfy $S_M(\phi(y),x)=2L(y,x)$.
When $H_M=2H_L$ as well, these reports embed the scaled loss $2L$ in the usual polyhedral sense.
These Bayes-risk and embedding statements in \citet[Theorem~4.7 and Appendix Theorem~B.9]{nowak2022consistency} are fully compatible with our examples.

The problem is the decoder.
For a report $v$, define its surrogate level set
\[
  \Gamma_v=\{q\in\DeltaY:v\in\Vopt(q)\},
\]
and for a target output $y$, define its Bayes region
\[
  Q_y=\{q\in\DeltaY:y\in\Bayes(q)\}.
\]

\begin{proposition}[Level-set reformulation for a prescribed decoder]
\label{prop:level-set-link}
Let $D(v)\subseteq\Y$ be a nonempty set-valued decoder.
Every surrogate-risk minimizer is decoded only to Bayes outputs if and only if
\begin{equation}
  \Gamma_v\subseteq\bigcap_{y\in D(v)}Q_y
  \qquad\text{for every }v\in\R^k.
  \label{eq:level-set-condition}
\end{equation}
For the canonical decoder, $D(v)=\argmax_yv_y$.
\end{proposition}

\begin{proof}
Condition \eqref{eq:level-set-condition} says exactly that whenever $v$ minimizes the conditional surrogate risk at $q$, every $y\in D(v)$ belongs to $\Bayes(q)$.
\end{proof}

An embedding verifies the required compatibility on the finite set $\phi(\Y)$.
It does not establish \eqref{eq:level-set-condition} for reports outside that image.
General polyhedral embedding theory instead constructs a suitable link on all reports \citep{finocchiaro2019embedding,finocchiaro2024embedding}; it does not assert that an arbitrary preassigned link, such as coordinate-wise argmax, is calibrated.

The unit star makes the distinction explicit.
At the distribution of \cref{cor:star}, the embedded center report
\[
  \phi(o)=(0,-1,-1,-1)
\]
is optimal and correctly decoded to $o$.
The additional optimal report
\[
  v=(-1,0,0,0)
\]
lies outside $\phi(\Y)$ and is decoded only to non-Bayes leaves.
Thus checking that argmax is an inverse of $\phi$ on $\phi(\Y)$ cannot imply consistency on the entire optimal set.

\begin{corollary}[Correction to the tree-metric claim]
Bayes-risk equality $H_M=2H_L$, even together with the corresponding scaled embedding, does not by itself imply consistency under coordinate-wise argmax.
Consequently, the argmax conclusions of Theorem~4.7 and Appendix Theorem~B.9 of \citet{nowak2022consistency}, and hence their Theorem~2.2 for tree metrics, fail for every branching tree under the canonical coordinate-wise argmax decoder.
This does not contradict those Bayes-risk and embedding statements, which can instead be paired with a separately constructed calibrated link.
\end{corollary}

\section{Discussion}

The common-geodesic condition is necessary but not sufficient for argmax Fisher consistency of the structured max-margin loss.
The obstruction occurs at the smallest possible output cardinality, persists throughout all branching tree metrics, and survives full-support restrictions on non-tree metrics.
Within tree metrics, the exact surviving class is the class of weighted paths; nevertheless, every tree retains the desired argmax property pointwise in the simplex interior.

The four-point star and the five-point $K_{2,3}$ example play complementary roles.
The star is minimal and directly corrects the claimed tree result, but uses a boundary distribution.
Among metrics satisfying the common-geodesic condition, the $K_{2,3}$ example is also the smallest possible full-support obstruction, and the family $K_{m,n}$ shows that it is not isolated.
The Hamming-cube example additionally shows that the phenomenon occurs for a standard coordinate-decomposable task metric.

Several questions remain.
First, a classification beyond tree metrics---for example, within finite median metric spaces---would identify exactly when the common-geodesic condition becomes sufficient.
Second, our five-point full-support witness has multiple Bayes outputs, whereas the Hamming-cube witness has a unique Bayes output and disjoint argmax; the sharp cardinality under these additional requirements remains open.
Third, the calibrated link supplied abstractly by polyhedral embedding theory may be describable explicitly for tree and median-type metrics, potentially retaining efficient loss-augmented inference while repairing prediction.

More broadly, the examples illustrate a useful warning for surrogate analysis:
matching Bayes risks determines optimal loss values and embedded optimal reports, but a learning method also commits to a decoder on every surrogate report.
Consistency of that decoder must be checked on the full collection of conditional optimal sets, including nonembedded faces created by polyhedral degeneracy.

\bibliographystyle{plainnat}
\bibliography{references}

\end{document}